\documentclass[journal]{IEEEtran}

\pdfoutput=1

\usepackage[utf8]{inputenc}

\title{Beyond $R$-barycenters: an effective averaging method on Stiefel and Grassmann manifolds}
\author{
\IEEEauthorblockN{Florent Bouchard, Nils Laurent, Salem Said, Nicolas Le Bihan}
\thanks{
Florent Bouchard is with Université Paris Saclay, CNRS, CentraleSupélec, laboratoire des signaux et systèmes.
Nils Laurent and Nicolas Le Bihan are with Université Grenoble Alpes, CNRS, Grenoble INP, Gipsa-lab.
Salem Said is with Université Grenoble Alpes, CNRS, Grenoble INP, laboratoire Jean Kuntzmann.
This work has been partially supported by MIAI @ Grenoble Alpes, (ANR-19-P3IA-0003).
}}

\usepackage{xcolor}

\usepackage{cite}

\usepackage[hidelinks]{hyperref}

\usepackage{units}
\usepackage{amsmath}
\usepackage{amssymb}
\usepackage{mathtools}
\usepackage{amsthm}
\usepackage{stmaryrd}

\usepackage{dsfont}

\usepackage{tikz}
\usepackage{pgfplots}
\pgfplotsset{compat=1.16}
\usepgfplotslibrary{fillbetween}

\newlength\height 
\newlength\width

\newcommand*{\MAT}[1]{\ensuremath{\boldsymbol{#1}}}

\newcommand*{\nfeatures}{\ensuremath{p}}
\newcommand*{\nrank}{\ensuremath{k}}
\newcommand*{\nsamples}{\ensuremath{n}}

\newcommand*{\realspace}{\ensuremath{\mathds{R}}}

\newcommand*{\St}{\ensuremath{\textup{St}_{\nfeatures,\nrank}}}
\newcommand*{\StAmbient}{\ensuremath{\realspace^{\nfeatures\times\nrank}}}
\newcommand*{\StTangent}[1]{\ensuremath{T_{#1}\St}}

\newcommand*{\Gr}{\ensuremath{\textup{Gr}_{\nfeatures,\nrank}}}
\newcommand*{\GrAmbient}{\ensuremath{\mathcal{S}_{\nfeatures}}}
\newcommand*{\GrTangent}[1]{\ensuremath{T_{#1}\Gr}}

\newcommand*{\manifold}{\ensuremath{\mathcal{M}}}
\newcommand*{\ambient}{\ensuremath{\mathcal{E}}}
\newcommand*{\tangentSpace}[1]{T_{#1}\manifold}

\newcommand*{\metric}[3]{\ensuremath{\langle #2, #3 \rangle_{#1}}}

\newcommand*{\eye}{\ensuremath{\MAT{I}}}

\DeclareMathOperator{\tr}{tr}
\DeclareMathOperator{\rank}{rank}
\DeclareMathOperator{\uf}{uf}
\DeclareMathOperator{\qf}{qf}
\DeclareMathOperator{\symm}{sym}
\DeclareMathOperator{\diff}{d}
\DeclareMathOperator{\tril}{tril}
\DeclareMathOperator{\expm}{expm}
\DeclareMathOperator{\logm}{logm}
\DeclareMathOperator{\err}{err}

\newcommand{\argmin}{\operatornamewithlimits{argmin}}

\newtheorem{definition}{Definition}
\newtheorem{proposition}{Proposition}

\definecolor{myblue}{HTML}{263b54}
\definecolor{myyellow}{HTML}{dea01e}
\definecolor{myred}{HTML}{E22623}
\definecolor{mylightblue}{HTML}{88B1FF}
\definecolor{mygreen}{HTML}{1FDE8B}
\definecolor{mypurple}{HTML}{A91FDE}
\usepackage{soul}

\begin{document}

\maketitle

\begin{abstract}
    In this paper, the issue of averaging data on a manifold is addressed.
    While the Fréchet mean resulting from Riemannian geometry appears ideal, it is unfortunately not always available and often computationally very expensive.
    To overcome this, $R$-barycenters have been proposed and successfully applied to Stiefel and Grassmann manifolds.
    However, $R$-barycenters still suffer severe limitations as they rely on iterative algorithms and complicated operators.
    We propose simpler, yet efficient, barycenters that we call $RL$-barycenters.
     We show that, in the setting relevant to most applications, our framework yields astonishingly simple barycenters: arithmetic means projected onto the manifold.
    We apply this approach to the Stiefel and Grassmann manifolds.
    On simulated data, our approach is competitive with respect to existing averaging methods, while computationally cheaper.
\end{abstract}


\section{Introduction}
\label{sec:intro}

In statistical signal processing and machine learning, it is often necessary to average data.
Indeed, this is for instance leveraged for classification (\textit{e.g.}, nearest centroid classifier~\cite{tuzel2008pedestrian, barachant2011multiclass}), clustering (\textit{e.g.}, $K$-means~\cite{arthur2007k}), shrinkage (to build the target matrix)~\cite{ledoit2004well, raninen2021linear}, batch normalization~\cite{santurkar2018does}, \textit{etc}.
When data possess a specific structure, \textit{e.g.}, when they belong to a smooth manifold, one should expect their average to possess the same structure and be adapted to the geometry of the manifold.
In such a case, the arithmetic mean is not well-suited.
Examples of such structured data are
covariance matrices, which are symmetric positive definite matrices (see, \textit{e.g.},~\cite{bouchard2024fisher} for a full review oriented on geometry);
orthogonal matrices, which are embedded in the Stiefel manifold~\cite{edelman1998geometry, absil2008optimization, boumal2023introduction};
or subspaces, which correspond to the Grassmann manifold~\cite{edelman1998geometry, absil2004riemannian, absil2008optimization, batzies2015geometric, boumal2023introduction, bendokat2024grassmann}.
While this letter aims to deal with generic smooth manifolds, a special attention is given to the Stiefel and Grassmann manifolds.
These are especially useful in the context of dimensionality reduction (see \textit{e.g.},~\cite{collas2021probabilistic} with an application to clustering) or deep learning~\cite{huang2023normalization}.

To average data on a smooth manifold, Riemannian geometry is often exploited.
Riemannian geometry indeed induces geodesics, which generalize the notion of straight lines, and a distance on the manifold; see, \textit{e.g.},~\cite{absil2008optimization, boumal2023introduction}.
These in turn lead to the definition of the Fréchet mean, which perfectly fits the geometry of the manifold.
While such a Fréchet mean appears ideal, it is unfortunately not always available and often computationally quite expensive.
Indeed, the distance is not always known in closed form -- \textit{e.g.}, for the Stiefel manifold -- or involves complicated operators -- such as the matrix logarithm for Grassmann~\cite{edelman1998geometry, absil2004riemannian, batzies2015geometric, bendokat2024grassmann}.
Even when available, an iterative algorithm is usually needed to compute the Fréchet mean; see \textit{e.g.},~\cite{said2017riemannian, bouchard2024fisher} for SPD matrices or~\cite{absil2004riemannian, batzies2015geometric} for Grassmann.
This algorithm relies on two objects: the Riemannian exponential, which maps tangent vectors onto the manifold following geodesics, and its inverse, the Riemannian logarithm.

To overcome the limitations of the Riemannian Fréchet mean,~\cite{kaneko2012empirical, fiori2014tangent} have proposed simpler averaging methods on manifolds: the so-called $R$-barycenters.
They are defined through a fixed-point equation that mimics the one that characterizes the Riemannian Fréchet mean.
The Riemannian exponential is replaced by a simpler tool: a retraction~\cite{absil2008optimization}, which can simply be a first order approximation of the Riemannian exponential.
The Riemannian logarithm is then replaced by the inverse of the chosen retraction.
This approach has been successfully applied on the Stiefel and Grassmann manifolds in~\cite{kaneko2012empirical, fiori2014tangent}.
While the $R$-barycenter framework is simpler than Riemannian Fréchet means, it still features major drawbacks.
Indeed, an iterative procedure is still needed and one has to combine a retraction with its exact inverse.
This second point appears as the most limiting one.
Indeed, for all considered retractions in~\cite{kaneko2012empirical, fiori2014tangent}, either the retraction or its inverse involves costly and possibly unstable operations.

In this letter, we follow a different path, recalling that the idea behind retractions is to simplify Riemannian exponentials.
Rather than choosing the inverse retraction to replace the Riemannian logarithm, we propose to leverage simpler liftings, which map points on the manifold onto tangent spaces, hence approximating the Riemannian logarithm.
This yields the so-called $RL$-barycenters.
Choosing the widely spread projection based retraction~\cite{absil2012projection} and the simplest lifting built on the Riemannian projection onto tangent spaces, we find out that the resulting $RL$-barycenter is astonishingly simple.
Indeed, it is just the projection onto the manifold of the arithmetic mean of the data.
Applied to the Stiefel manifold, we show that the resulting barycenter is in fact a closed form solution of the $R$-barycenter associated to the orthographic retraction from~\cite{kaneko2012empirical}.
We also extend our result to the projection based on QR decomposition, showing that the resulting projected mean is also an $RL$-barycenter.
In order to apply our approach on the Grassmann manifold, we derive the projection from the ambient space onto the manifold.
Numerical experiments are conducted on simulated data.
Our projected means perform better than existing $R$-barycenters on Stiefel.
We also do not lose too much accuracy as compared to the Riemannian Fréchet mean on Grassmann.
Due to their simplicity and reasonable complexity, on Stiefel and Grassmann manifolds, our proposed projected means appear very advantageous as compared to other existing averaging methods.

To ensure reproducibility, the code is available at \url{https://github.com/flbouchard/projection_barycenter}.

\section{Background}
\label{sec:background}
\subsection{Stiefel and Grassmann manifolds}

The real Stiefel manifold is the homogeneous space of $\nfeatures\times\nrank$ orthogonal matrices~\cite{edelman1998geometry, absil2008optimization, boumal2023introduction}, \textit{i.e.},
\begin{equation}
    \St = \{ \MAT{U}\in\StAmbient: \; \MAT{U}^\top\MAT{U} = \eye_\nrank \}.
\label{eq:st}
\end{equation}
The projection map from $\StAmbient$ onto $\St$ according to the Euclidean distance is~\cite[Theorem 4.1]{higham1988matrix}
\begin{equation}
    \mathcal{P}^{\St}(\MAT{X}) = \argmin_{\MAT{U}\in\St} \; \| \MAT{X} - \MAT{U} \|_2^2 = \uf(X),
\label{eq:st_proj}
\end{equation}
where $\uf(\cdot)$ returns the orthogonal factor of the polar decomposition.
The tangent space of $\St$ at $\MAT{U}$ is~\cite{edelman1998geometry, absil2008optimization, boumal2023introduction}
\begin{equation}
    \StTangent{\MAT{U}} = \{ \MAT{\xi}\in\StAmbient: \; \MAT{U}^\top\MAT{\xi} + \MAT{\xi}^\top\MAT{U} = \MAT{0} \}.
\label{eq:st_tangent}
\end{equation}
Since $\St$ is a submanifold of the Euclidean space $\StAmbient$, it can simply be turned into a Riemannian manifold by endowing it with the Euclidean metric
\begin{equation}
    \metric{\MAT{U}}{\MAT{\xi}}{\MAT{\eta}} = \tr(\MAT{\xi}^\top\MAT{\eta}).
\label{eq:eucl_metric}
\end{equation}
The corresponding orthogonal projection from $\StAmbient$ onto $\StTangent{\MAT{U}}$ is~\cite{edelman1998geometry, absil2008optimization, boumal2023introduction}
\begin{equation}
    P^{\St}_{\MAT{U}}(\MAT{Z}) = \MAT{Z} - \MAT{U}\symm(\MAT{U}^\top\MAT{Z}),
\label{eq:st_tangent_proj}
\end{equation}
where $\symm(\cdot)$ returns the symmetrical part of its argument.

The Grassmann manifold is the manifold of $\nrank$-dimensional subspaces in the Euclidean space $\realspace^\nfeatures$~\cite{edelman1998geometry, absil2004riemannian, absil2008optimization, batzies2015geometric, boumal2023introduction, bendokat2024grassmann}.
There exist various ways of representing it.
For instance, it can be viewed as a quotient manifold of the Stiefel manifold $\St$ with the orthogonal group $\mathcal{O}_{\nrank}$~\cite{edelman1998geometry, absil2004riemannian, absil2008optimization, batzies2015geometric, boumal2023introduction, bendokat2024grassmann}.
In this article, as in~\cite{batzies2015geometric, bendokat2024grassmann}, we identify it with the set of orthogonal rank $\nrank$ projectors, \textit{i.e.},
\begin{equation}
    \Gr = \{ \MAT{P}\in\GrAmbient: \; \MAT{P}^2=\MAT{P}, \; \rank(\MAT{P}) = \nrank \},
\label{eq:gr}
\end{equation}
where $\GrAmbient$ denotes the Euclidean space of $\nfeatures\times\nfeatures$ symmetric matrices.
This representation of the Grassmann manifold $\Gr$ is linked to the Stiefel manifold $\St$ through the projection mapping
\begin{equation}
    \pi : \MAT{U}\in\St \mapsto \MAT{U}\MAT{U}^\top\in\Gr.
\label{eq:st_to_gr}
\end{equation}
Even though the formula is quite intuitive and related to principal component analysis, we could not find the projection map from $\GrAmbient$ onto $\Gr$ identified as~\eqref{eq:gr} in the literature.
We thus provide it in Section~\ref{sec:barycenters}, which contains our contributions.
%
%
The tangent space of the Grassmann manifold identified as~\eqref{eq:gr} at $\MAT{P}\in\Gr$ is~\cite{bendokat2024grassmann}
\begin{equation}
    \GrTangent{\MAT{P}} = \{ \MAT{\xi}\in\GrAmbient: \; \MAT{P}\MAT{\xi} + \MAT{\xi}\MAT{P} = \MAT{\xi} \}.
\end{equation}
Since, in this case, $\Gr$ is a submanifold of $\GrAmbient$, it can also be turned into a Riemannian manifold by endowing it with the Euclidean metric~\eqref{eq:eucl_metric}.
The corresponding orthogonal projection from $\GrAmbient$ onto $\GrTangent{\MAT{P}}$ is~\cite{bendokat2024grassmann}
\begin{equation}
    P^{\Gr}_{\MAT{P}}(\MAT{Z}) = 2\symm((\eye_{\nfeatures} - \MAT{P})\MAT{Z}\MAT{P}).
\label{eq:gr_tangent_proj}
\end{equation}

\subsection{Barycenters on matrix manifolds}

When aiming to compute a barycenter on a Riemannian matrix manifold $\manifold$, the ideal solution appears to employ the Riemannian mean.
Such manifold is equipped with a Riemannian metric $\metric{\cdot}{\cdot}{\cdot}$, which yields a Riemannian distance $\delta(\cdot,\cdot)$ on $\manifold$.
This distance can be exploited to define the corresponding Riemannian mean (or Fréchet mean).
Given samples $\{\MAT{M}_i\}_{i=1}^\nsamples$ in $\manifold$, their Riemannian mean $\MAT{G}\in\manifold$ is the solution to the optimization problem~\cite{afsari2013convergence}
\begin{equation}
    \MAT{G} = \argmin_{\MAT{G}\in\manifold} \quad \sum_{i=1}^\nsamples \delta^2(\MAT{M}_i,\MAT{G}).
\end{equation}
It is usually not known in closed form.
To compute it, one can employ the Riemannian gradient descent, which yields the following fixed-point algorithm~\cite{manton2004globally,afsari2013convergence}
\begin{equation}
    \MAT{G}^{(t+1)} = \exp_{\MAT{G}^{(t)}} \left( \frac{1}{\nsamples}\sum_{i=1}^{\nsamples} \log_{\MAT{G}^{(t)}}(\MAT{M}_i) \right),
\label{eq:Rmean_algo}
\end{equation}
where $\exp_{\MAT{G}}:\tangentSpace{\MAT{G}}\to\manifold$ and $\log_{\MAT{G}}:\manifold\to\tangentSpace{\MAT{G}}$ are the Riemannian exponential and logarithm at $\MAT{G}\in\manifold$.
The Riemannian exponential is defined through the geodesics, which generalize the notion of straight lines to Riemannian manifolds.
The Riemannian logarithm is its (local) inverse.

Unfortunately, even though it seems the most natural option, the Riemannian mean is often very complicated to compute in practice.
This is because Riemannian exponential and logarithm operators are computationally expensive in many cases.
In fact, they are not always known in closed form (especially the Riemannian logarithm) and, even when they are, their computation usually involves costly operations.
For instance, for the Stiefel manifold, the Riemannian exponential involves a matrix exponential~\cite{edelman1998geometry, absil2008optimization, zimmermann2022computing} while the Riemannian logarithm is not known in closed form and can only be computed with a heavy iterative algorithm~\cite{zimmermann2017matrix, zimmermann2022computing, mataigne2024efficient}.

To overcome the fact that the Riemannian exponential is often too expensive, a simpler tool to map tangent vectors onto the manifold has been designed in the context of optimization: the retraction~\cite{absil2008optimization}.
A retraction is, at $\MAT{G}\in\manifold$, a mapping $R_{\MAT{G}}:\tangentSpace{\MAT{G}}\to\manifold$ such that $R_{\MAT{G}}(\MAT{\xi})=\MAT{G} + \MAT{\xi} + o(\| \MAT{\xi} \|)$.
Retractions are (at least) first order approximations of the Riemannian exponential.
Notice that on a manifold, there are often several retractions available.
Beyond optimization, retractions have been leveraged to design barycenters on manifolds: the so-called $R$-barycenters~\cite{kaneko2012empirical, fiori2014tangent}.
The goal is to propose simpler barycenters than the Riemannian mean while respecting the structure of the manifold.
This appears particularly attractive for manifolds whose Riemannian exponential and/or logarithm are not known in closed form such as the Stiefel manifold.
The idea is to mimic~\eqref{eq:Rmean_algo}, replacing the Riemannian exponential and logarithm with a retraction and its inverse~\cite{kaneko2012empirical, fiori2014tangent}.
Formally, the resulting fixed-point algorithm is
\begin{equation}
    \MAT{G}^{(t+1)} = R_{\MAT{G}^{(t)}} \left( \frac{1}{\nsamples}\sum_{i=1}^{\nsamples} R_{\MAT{G}^{(t)}}^{-1}(\MAT{M}_i) \right).
\label{eq:Rbarycenter_algo}
\end{equation}
In practice, this approach has been exploited on the Stiefel manifold with various retractions~\cite{kaneko2012empirical}.
The first one is the one based on the projection~\eqref{eq:st_proj} (polar decomposition), \textit{i.e.},
\begin{equation}
    R^{\uf}_{\MAT{U}}(\MAT{\xi}) = \mathcal{P}^{\St}(\MAT{U}+\MAT{\xi}) = \uf(\MAT{U}+\MAT{\xi}).
\label{eq:st_retr_polar}
\end{equation}
The second one is based on the QR decomposition, \textit{i.e.},
\begin{equation}
    R^{\qf}_{\MAT{U}}(\MAT{\xi}) = \qf(\MAT{U}+\MAT{\xi}),
\label{eq:st_retr_qr}
\end{equation}
where $\qf(\cdot)$ returns the orthogonal factor of the QR decomposition.
For these two retractions, computing the inverse is not straightforward.
In both case, it involves solving equations not admitting closed form solutions.
The third retraction is the so-called orthographic retraction~\cite{kaneko2012empirical}.
This has a straightforward inverse, while the retraction itself is implicitly
defined and involves solving a Ricatti equation.
The inverse retraction exploits the orthogonal projection~\eqref{eq:st_tangent_proj} and is given by
\begin{equation}
    R^{\textup{o} \; -1}_{\MAT{U}}(\MAT{V}) = P^{\St}_{\MAT{U}}(\MAT{V}-\MAT{U}).
\label{eq:st_retr_orthographic}
\end{equation}

For all above $R$-barycenters, a simple expression exists either for the retraction $R_{\cdot}(\cdot)$ or the inverse retraction $R_{\cdot}^{-1}(\cdot)$, but numerically solving an equation, possibly costly and unstable, is necessary for the other operation.
Indeed, as explained in~\cite{kaneko2012empirical}, a solution to such equation is only guaranteed in a neighborhood of $\MAT{U}\in\St$.
Hence, the resulting procedure~\eqref{eq:Rbarycenter_algo} appears quite complicated and heavy.
Moreover, the motivation behind retractions is to simplify the Riemannian exponential.
Exactly taking the inverse retraction, which is complicated, does not seem to follow this philosophy.

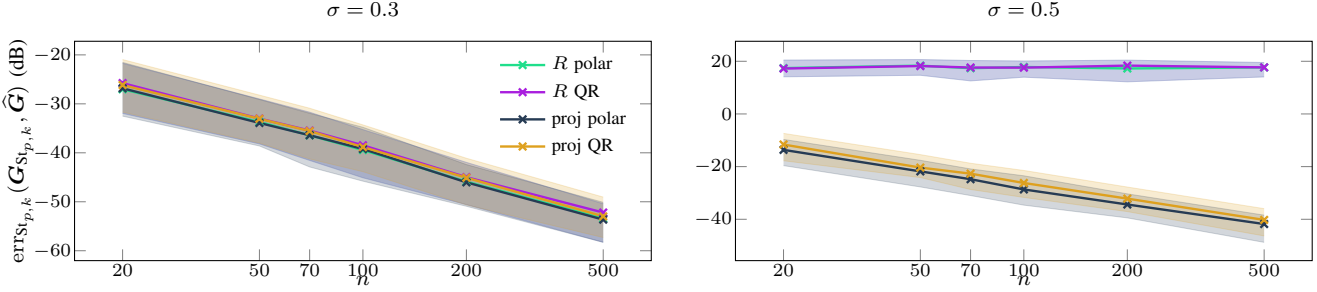
\begin{figure*}[t]
    \setlength\height{4.5cm} 
\setlength\width{0.51\linewidth}

\begin{tikzpicture}

\begin{semilogxaxis}[
    width  =\width,
    height =\height,
    at     ={(0,0)},
    xlabel = {$\nsamples$},
    xlabel style = {font=\footnotesize,yshift=1.5ex},
    xtick = {20,50,70,100,200,500},
    xticklabels = {$20$,$50$,$70$,$100$,$200$,$500$},
    xticklabel style = {font=\scriptsize,yshift=0.5ex},
    ylabel = {$\err_{\St}(\MAT{G}_{\St},\MAT{\widehat{G}})$~(dB)},
    ylabel style = {font=\footnotesize,yshift=-1.5ex},
    yticklabel style = {font=\scriptsize,xshift=0.5ex},
    title = {\footnotesize$\sigma=0.3$},
    legend style={legend cell align=left,align=left,draw=none,fill=none,font=\scriptsize,legend columns=1,transpose legend}
]
    
    \addplot[color=mygreen,opacity=0.2,line width=0.1pt, forget plot, name path=A] table [x=n,y expr=20*log10(\thisrow{R_polar_q10}),col sep=comma] {./figures/error_st_p10_k5_scale03_nMC100.csv};
    \addplot[color=mygreen,opacity=0.2,line width=0.1pt, forget plot, name path=B] table [x=n,y expr=20*log10(\thisrow{R_polar_q90}),col sep=comma] {./figures/error_st_p10_k5_scale03_nMC100.csv};
    \addplot[color=mygreen,opacity=0.2, forget plot] fill between[of=A and B];
    
    \addplot[color=mypurple,opacity=0.2,line width=0.1pt, forget plot, name path=A] table [x=n,y expr=20*log10(\thisrow{R_polar_q10}),col sep=comma] {./figures/error_st_p10_k5_scale03_nMC100.csv};
    \addplot[color=mypurple,opacity=0.2,line width=0.1pt, forget plot, name path=B] table [x=n,y expr=20*log10(\thisrow{R_polar_q90}),col sep=comma] {./figures/error_st_p10_k5_scale03_nMC100.csv};
    \addplot[color=mypurple,opacity=0.2, forget plot] fill between[of=A and B];
    
    \addplot[color=myblue,opacity=0.2,line width=0.1pt, forget plot, name path=A] table [x=n,y expr=20*log10(\thisrow{proj_polar_q10}),col sep=comma] {./figures/error_st_p10_k5_scale03_nMC100.csv};
    \addplot[color=myblue,opacity=0.2,line width=0.1pt, forget plot, name path=B] table [x=n,y expr=20*log10(\thisrow{proj_polar_q90}),col sep=comma] {./figures/error_st_p10_k5_scale03_nMC100.csv};
    \addplot[color=myblue,opacity=0.2, forget plot] fill between[of=A and B];
    
    \addplot[color=myyellow,opacity=0.2,line width=0.1pt, forget plot, name path=A] table [x=n,y expr=20*log10(\thisrow{proj_qr_q10}),col sep=comma] {./figures/error_st_p10_k5_scale03_nMC100.csv};
    \addplot[color=myyellow,opacity=0.2,line width=0.1pt, forget plot, name path=B] table [x=n,y expr=20*log10(\thisrow{proj_qr_q90}),col sep=comma] {./figures/error_st_p10_k5_scale03_nMC100.csv};
    \addplot[color=myyellow,opacity=0.2, forget plot] fill between[of=A and B];

    \addplot[color=mygreen,line width=0.9pt,mark=x,mark size=2pt] table [x=n,y expr=20*log10(\thisrow{R_polar_median}),col sep=comma] {./figures/error_st_p10_k5_scale03_nMC100.csv};
    \addlegendentry{$R$ polar};
    
    \addplot[color=mypurple,line width=0.9pt,mark=x,mark size=2pt] table [x=n,y expr=20*log10(\thisrow{R_qr_median}),col sep=comma] {./figures/error_st_p10_k5_scale03_nMC100.csv};
    \addlegendentry{$R$ QR};
    
    \addplot[color=myblue,line width=0.9pt,mark=x,mark size=2pt] table [x=n,y expr=20*log10(\thisrow{proj_polar_median}),col sep=comma] {./figures/error_st_p10_k5_scale03_nMC100.csv};
    \addlegendentry{proj polar};
    
    \addplot[color=myyellow,line width=0.9pt,mark=x,mark size=2pt] table [x=n,y expr=20*log10(\thisrow{proj_qr_median}),col sep=comma] {./figures/error_st_p10_k5_scale03_nMC100.csv};
    \addlegendentry{proj QR};

\end{semilogxaxis}

\begin{semilogxaxis}[
    width  =\width,
    height =\height,
    at     ={(0.95\width,0)},
    xlabel = {$\nsamples$},
    xlabel style = {font=\footnotesize,yshift=1.5ex},
    xtick = {20,50,70,100,200,500},
    xticklabels = {$20$,$50$,$70$,$100$,$200$,$500$},
    xticklabel style = {font=\scriptsize,yshift=0.5ex},
    yticklabel style = {font=\scriptsize,xshift=0.5ex},
    title = {\footnotesize$\sigma=0.5$},
]
    
    \addplot[color=mygreen,opacity=0.2,line width=0.1pt, forget plot, name path=A] table [x=n,y expr=20*log10(\thisrow{R_polar_q10}),col sep=comma] {./figures/error_st_p10_k5_scale05_nMC100.csv};
    \addplot[color=mygreen,opacity=0.2,line width=0.1pt, forget plot, name path=B] table [x=n,y expr=20*log10(\thisrow{R_polar_q90}),col sep=comma] {./figures/error_st_p10_k5_scale05_nMC100.csv};
    \addplot[color=mygreen,opacity=0.2, forget plot] fill between[of=A and B];
    
    \addplot[color=mypurple,opacity=0.2,line width=0.1pt, forget plot, name path=A] table [x=n,y expr=20*log10(\thisrow{R_polar_q10}),col sep=comma] {./figures/error_st_p10_k5_scale05_nMC100.csv};
    \addplot[color=mypurple,opacity=0.2,line width=0.1pt, forget plot, name path=B] table [x=n,y expr=20*log10(\thisrow{R_polar_q90}),col sep=comma] {./figures/error_st_p10_k5_scale05_nMC100.csv};
    \addplot[color=mypurple,opacity=0.2, forget plot] fill between[of=A and B];
    
    \addplot[color=myblue,opacity=0.2,line width=0.1pt, forget plot, name path=A] table [x=n,y expr=20*log10(\thisrow{proj_polar_q10}),col sep=comma] {./figures/error_st_p10_k5_scale05_nMC100.csv};
    \addplot[color=myblue,opacity=0.2,line width=0.1pt, forget plot, name path=B] table [x=n,y expr=20*log10(\thisrow{proj_polar_q90}),col sep=comma] {./figures/error_st_p10_k5_scale05_nMC100.csv};
    \addplot[color=myblue,opacity=0.2, forget plot] fill between[of=A and B];
    
    \addplot[color=myyellow,opacity=0.2,line width=0.1pt, forget plot, name path=A] table [x=n,y expr=20*log10(\thisrow{proj_qr_q10}),col sep=comma] {./figures/error_st_p10_k5_scale05_nMC100.csv};
    \addplot[color=myyellow,opacity=0.2,line width=0.1pt, forget plot, name path=B] table [x=n,y expr=20*log10(\thisrow{proj_qr_q90}),col sep=comma] {./figures/error_st_p10_k5_scale05_nMC100.csv};
    \addplot[color=myyellow,opacity=0.2, forget plot] fill between[of=A and B];

    \addplot[color=mygreen,line width=0.9pt,mark=x,mark size=2pt] table [x=n,y expr=20*log10(\thisrow{R_polar_median}),col sep=comma] {./figures/error_st_p10_k5_scale05_nMC100.csv};
    
    \addplot[color=mypurple,line width=0.9pt,mark=x,mark size=2pt] table [x=n,y expr=20*log10(\thisrow{R_qr_median}),col sep=comma] {./figures/error_st_p10_k5_scale05_nMC100.csv};
    
    \addplot[color=myblue,line width=0.9pt,mark=x,mark size=2pt] table [x=n,y expr=20*log10(\thisrow{proj_polar_median}),col sep=comma] {./figures/error_st_p10_k5_scale05_nMC100.csv};
    
    \addplot[color=myyellow,line width=0.9pt,mark=x,mark size=2pt] table [x=n,y expr=20*log10(\thisrow{proj_qr_median}),col sep=comma] {./figures/error_st_p10_k5_scale05_nMC100.csv};

\end{semilogxaxis}

\end{tikzpicture}
    \vspace*{-10pt}
    \caption{Medians (solid lines), 10\% and 90\% quantiles (filled areas) over $100$ realizations of error measure~\eqref{eq:err_st} of mean estimators on the Stiefel manifold $\St$.
    ``$R$ polar'' and ``$R$ QR'' correspond to $R$-barycenters with polar and QR retractions.
    ``proj polar'' and ``proj QR'' correspond to the projected arithmetic means with the projections on $\St$ based on the polar and QR decompositions, respectively.
    In these simulations, $p=10$ and $k=5$.}
    \label{fig:stiefel}
\end{figure*}

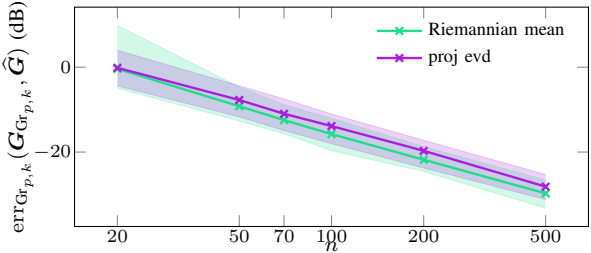
\begin{figure}[t]
    \setlength\height{4.5cm} 
\setlength\width{0.95\linewidth}

\begin{tikzpicture}

\begin{semilogxaxis}[
    width  =\width,
    height =\height,
    at     ={(0,0)},
    xlabel = {$\nsamples$},
    xlabel style = {font=\footnotesize,yshift=1.5ex},
    xtick = {20,50,70,100,200,500},
    xticklabels = {$20$,$50$,$70$,$100$,$200$,$500$},
    xticklabel style = {font=\scriptsize,yshift=0.5ex},
    ylabel = {$\err_{\Gr}(\MAT{G}_{\Gr},\MAT{\widehat{G}})$~(dB)},
    ylabel style = {font=\footnotesize,yshift=-1.5ex},
    yticklabel style = {font=\scriptsize,xshift=0.5ex},
    legend style={legend cell align=left,align=left,draw=none,fill=none,font=\scriptsize,legend columns=1,transpose legend}
]

    \addplot[color=mygreen,opacity=0.2,line width=0.1pt, forget plot, name path=A] table [x=n,y expr=20*log10(\thisrow{Riem_mean_q10}),col sep=comma] {./figures/error_gr_p10_k5_scale05_nMC100.csv};
    \addplot[color=mygreen,opacity=0.2,line width=0.1pt, forget plot, name path=B] table [x=n,y expr=20*log10(\thisrow{Riem_mean_q90}),col sep=comma] {./figures/error_gr_p10_k5_scale05_nMC100.csv};
    \addplot[color=mygreen,opacity=0.2, forget plot] fill between[of=A and B];
    
    \addplot[color=mypurple,opacity=0.2,line width=0.1pt, forget plot, name path=A] table [x=n,y expr=20*log10(\thisrow{proj_evd_q10}),col sep=comma] {./figures/error_gr_p10_k5_scale05_nMC100.csv};
    \addplot[color=mypurple,opacity=0.2,line width=0.1pt, forget plot, name path=B] table [x=n,y expr=20*log10(\thisrow{proj_evd_q90}),col sep=comma] {./figures/error_gr_p10_k5_scale05_nMC100.csv};
    \addplot[color=mypurple,opacity=0.2, forget plot] fill between[of=A and B];

    \addplot[color=mygreen,line width=0.9pt,mark=x,mark size=2pt] table [x=n,y expr=20*log10(\thisrow{Riem_mean_median}),col sep=comma] {./figures/error_gr_p10_k5_scale05_nMC100.csv};
    \addlegendentry{Riemannian mean};
    
    \addplot[color=mypurple,line width=0.9pt,mark=x,mark size=2pt] table [x=n,y expr=20*log10(\thisrow{proj_evd_median}),col sep=comma] {./figures/error_gr_p10_k5_scale05_nMC100.csv};
    \addlegendentry{proj evd};

\end{semilogxaxis}

\end{tikzpicture}
    \vspace*{-10pt}
    \caption{Medians (solid lines), 10\% and 90\% quantiles (filled areas) over $100$ realizations of error measure~\eqref{eq:err_gr} of mean estimators on the Grassmann manifold $\Gr$.
    In the legend, ``proj evd'' corresponds to the projected arithmetic mean with the projection on $\Gr$ based on the eigenvalue decomposition.
    In these simulations, $p=10$, $k=5$ and $\sigma=0.5$.}
    \label{fig:grassmann}
\end{figure}

\section{Projection based barycenters}
\label{sec:barycenters}
This section contains our contribution.
Our original idea is to simplify~\eqref{eq:Rbarycenter_algo} by dropping the requirement of choosing the inverse retraction.
We rather replace that with a lifting, which, at $\MAT{G}\in\manifold$, is a mapping $L_{\MAT{G}}:\manifold\to\tangentSpace{\MAT{G}}$ such that $L_{\MAT{G}}(\MAT{M})=\MAT{M}-\MAT{G}+o(\|\MAT{M}\|)$.
The resulting barycenters, named retraction-lifting barycenters, and denoted $RL$-barycenters, are defined in Definition~\ref{def:RLbarycenter}.

\begin{definition}[$RL$-barycenters]
    \label{def:RLbarycenter}
    Given the retraction $R_{\cdot}:\tangentSpace{\cdot}\to\manifold$ and lifting $L_{\cdot}:\manifold\to\tangentSpace{\cdot}$, the so-called $RL$-barycenter $\MAT{G}\in\manifold$ of samples $\{\MAT{M}_i\}_{i=1}^\nsamples$ in $\manifold$, if it exists, is solution to the fixed-point equation
    \begin{equation*}
        \MAT{G} = R_{\MAT{G}}\left( \frac{1}{\nsamples}\sum_{i=1}^{\nsamples} L_{\MAT{G}}(\MAT{M}_i) \right).
    \end{equation*}
    Notice that the point $\MAT{G}\in\manifold$ is solution if it verifies $\frac{1}{\nsamples}\sum_{i=1}^{\nsamples} L_{\MAT{G}}(\MAT{M}_i) = \MAT{0}$.
\end{definition}

This formalism of $RL$-barycenters encompasses the existing ones of Riemannian means -- with $R_{\MAT{G}}(\cdot)=\exp_{\MAT{G}}(\cdot)$ and $L_{\MAT{G}}(\cdot)=\log_{\MAT{G}}(\cdot)$ -- and $R$-barycenters -- with $L_{\MAT{G}}(\cdot)=R_{\MAT{G}}^{-1}(\cdot)$.
It is more general since a wider range of choices of liftings is possible.
Hence, it allows to select retractions and liftings known in closed form and not too expensive to compute.
One can thus expect to obtain more tractable algorithms with this setting.
Choosing simple yet natural retraction and lifting is our goal in the following.
As we will see, doing so yields astonishingly simple barycenters.

In this work, we are particularly interested in the retractions that arise from the projection from the ambient space $\ambient$ to the matrix manifold $\manifold$~\cite{absil2012projection}, defined as $\mathcal{P}(\MAT{X}) = \argmin_{\MAT{G}\in\manifold} \, \| \MAT{X} - \MAT{G} \|_2^2$.
The corresponding retraction is
\begin{equation}
    R_{\MAT{G}}(\MAT{\xi}) = \mathcal{P}(\MAT{G}+\MAT{\xi}).
\label{eq:retr_proj}
\end{equation}
For the lifting, we consider the orthogonal projection mapping on tangent spaces corresponding to the Euclidean metric of $\ambient$.
At $\MAT{G}\in\manifold$, it is denoted $P_{\MAT{G}}:\ambient\to\tangentSpace{\MAT{G}}$.
The lifting is
\begin{equation}
    L_{\MAT{G}}(\MAT{M}) = P_{\MAT{G}}(\MAT{M}-\MAT{G}).
\label{eq:lift_proj}
\end{equation}
This retraction and lifting appear as the simplest natural choices on $\manifold$.
Interestingly, as shown in Proposition~\ref{prop:proj_barycenter}, we soon realised that the resulting barycenter admits a simple closed form expression: it is the projection on $\manifold$ of the arithmetic mean of $\{\MAT{M}_i\}_{i=1}^{\nsamples}$ (which belongs to $\ambient$).

\begin{proposition}[Projection based barycenters]
    \label{prop:proj_barycenter}
    Given the retraction~\eqref{eq:retr_proj} and the lifting~\eqref{eq:lift_proj}, the $RL$-barycenter of $\{\MAT{M}_i\}_{i=1}^{\nsamples}$, according to Definition~\ref{def:RLbarycenter}, is 
    \begin{equation*}
        \MAT{G} = \mathcal{P}\left( \frac{1}{\nsamples}\sum_{i=1}^{\nsamples} \MAT{M}_i \right).
    \end{equation*}
\end{proposition}
\begin{proof}
    By definition, $\MAT{G} = \argmin_{\MAT{G}\in\manifold} \; \| \frac{1}{\nsamples}\sum_{i=1}^{\nsamples} \MAT{M}_i - \MAT{G} \|_2^2$.
    Let $F(\MAT{G})=\| \frac{1}{\nsamples}\sum_{i=1}^{\nsamples} \MAT{M}_i - \MAT{G} \|_2^2$.
    The directional derivative of $F$ at $\MAT{G}\in\manifold$ in direction $\MAT{\xi}\in\tangentSpace{\MAT{G}}$ is $\diff F(\MAT{G})[\MAT{\xi}]=\langle \MAT{G} - \frac{1}{\nsamples} \sum_{i=1}^{\nsamples}\MAT{M}_i , \MAT{\xi} \rangle$, where $\langle\cdot,\cdot\rangle$ denotes the Euclidean metric on $\ambient$.
    Since $\MAT{\xi}\in\tangentSpace{\MAT{G}}$, one has 
    \begin{equation*}
        \diff F(\MAT{G})[\MAT{\xi}]
        \textstyle
        = \langle P_{\MAT{G}}(\MAT{G} - \frac{1}{\nsamples} \sum_{i=1}^{\nsamples}\MAT{M}_i) , \MAT{\xi} \rangle.
    \end{equation*}
    By identification, it follows that the Riemannian gradient of $F$ at $\MAT{G}$ is $\nabla F(\MAT{G}) = P_{\MAT{G}}(\MAT{G} - \frac{1}{\nsamples} \sum_{i=1}^{\nsamples}\MAT{M}_i)$.
    Moreover, by definition of the projected mean $\MAT{G}$, $\nabla F(\MAT{G})=\MAT{0}$.
    Hence, $P_{\MAT{G}}(\frac{1}{\nsamples} \sum_{i=1}^{\nsamples} (\MAT{M}_i - \MAT{G}))=\frac{1}{\nsamples}\sum_{i=1}^{\nsamples} L_{\MAT{G}}(\MAT{M}_i)=\MAT{0}$.
\end{proof}

In particular, this approach can be employed with the Stiefel manifold $\St$ with the retraction and lifting resulting from projections~\eqref{eq:st_proj} and~\eqref{eq:st_tangent_proj}.
%
It is interesting to notice that both the retraction and lifting were previously considered in the context of $R$-barycenters.
Indeed, the retraction corresponds to the polar retraction~\eqref{eq:st_retr_polar} while the lifting corresponds to the inverse retraction~\eqref{eq:st_retr_orthographic} of the orthographic retraction.
One of the results of the present paper is that, as a direct consequence of Proposition~\ref{prop:proj_barycenter}, $\MAT{G}=\uf(\frac{1}{\nsamples}\sum_{i=1}^{\nsamples} \MAT{M}_i)$ is a closed form solution for the $R$-barycenter with the orthographic retraction.
Hence, in this case, the iterative procedure~\eqref{eq:Rbarycenter_algo} is no longer necessary.
To apply our approach on the Grassmann manifold $\Gr$, the projection map from $\GrAmbient$ onto $\Gr$ is required. It is provided in Proposition~\ref{prop:gr_proj}.

\begin{proposition}[Projection on the Grassmann manifold]
    The projection map from $\GrAmbient$ onto $\Gr$ according to the Euclidean distance is
    \begin{equation*}
        \mathcal{P}^{\Gr}(\MAT{X}) = \argmin_{\MAT{P}\in\Gr} \; \| \MAT{X} - \MAT{P} \|_2^2 = \MAT{V}_{\nrank}\MAT{V}_{\nrank}^\top,
    \label{eq:gr_proj}
    \end{equation*}
    where $\MAT{V}_{\nrank}$ is composed of the $\nrank$ eigenvectors corresponding to the $\nrank$ largest eigenvalues of $\MAT{X}$. 
    \label{prop:gr_proj}
\end{proposition}
\begin{proof}
    See Supplementary materials.
\end{proof}

We further believe that our results extend to more generic projections, \textit{i.e.}, mappings $\widetilde{\mathcal{P}}:\ambient\to\manifold$ such that $\widetilde{\mathcal{P}}^2(\MAT{X})=\widetilde{\mathcal{P}}(\MAT{X})$. 
From~\cite{absil2012projection}, we know that, in the case of $\mathcal{P} : \MAT{X}\in\ambient \mapsto \argmin_{\MAT{G}\in\manifold} \; \|\MAT{X}-\MAT{G}\|_2^2$, we have $P_{\MAT{G}}(\MAT{Z})=\diff\mathcal{P}(\MAT{G})[\MAT{Z}]$.
Hence, if we set $R_{\MAT{G}}(\MAT{\xi})=\widetilde{\mathcal{P}}(\MAT{G}+\MAT{\xi})$ and $L_{\MAT{G}}(\MAT{M})=\diff\widetilde{\mathcal{P}}(\MAT{G})[\MAT{M}-\MAT{G}]$, then the corresponding $RL$-barycenter is the arithmetic mean projected on $\manifold$ with $\widetilde{\mathcal{P}}$, \textit{i.e.}, $\MAT{G}=\widetilde{\mathcal{P}}(\frac{1}{\nsamples}\sum_{i=1}^{\nsamples} \MAT{M}_i)$.
To obtain this, it is needed to show that, with $\MAT{G}=\widetilde{\mathcal{P}}(\frac{1}{\nsamples}\sum_{i=1}^{\nsamples} \MAT{M}_i)$, we have $\diff\widetilde{\mathcal{P}}(\MAT{G})[\frac{1}{\nsamples}\sum_{i=1}^{\nsamples} \MAT{M}_i - \MAT{G}]=\MAT{0}$.
Intuitively, this seems to be the case but proving it is beyond the scope of the present letter in the general case.
In supplementary materials, we show that this actually works on $\St$ with the projection based on the QR decomposition, \textit{i.e.},
\begin{equation}
    \MAT{G} = \qf\left( \frac{1}{\nsamples}\sum_{i=1}^{\nsamples} \MAT{M}_i \right)
    \label{eq:st_proj_barycenter_QR}
\end{equation}
is the $RL$-barycenter of $\{ \MAT{M}_i \}_{i=1}^{\nsamples}$ with the QR retraction~\eqref{eq:st_retr_qr} and the lifting $L_{\MAT{G}}(\MAT{M})=\diff\qf(\MAT{G})[\MAT{M}-\MAT{G}]$.

\section{Numerical experiments}
\label{sec:num_exp}

In this section, numerical experiments on simulated data are conducted to evaluate the performance of the proposed projected arithmetic means from Proposition~\ref{prop:proj_barycenter} on Stiefel and Grassmann manifolds.
The projected mean~\eqref{eq:st_proj_barycenter_QR} based on the QR decomposition is also considered.
For the Stiefel manifold $\St$, the performance of proposed means are compared to the ones of $R$-barycenters~\cite{kaneko2012empirical} exploiting the polar~\eqref{eq:st_retr_polar}, QR~\eqref{eq:st_retr_qr} and orthographic~\eqref{eq:st_retr_orthographic} retractions.
For the Grassmann manifold, the projected mean is compared to the Riemannian mean; see, \textit{e.g.},~\cite{absil2004riemannian, batzies2015geometric}.
In every cases, iterative algorithms are initialized with the first sample of the dataset to average.

Let us now describe how simulated data are obtained.
For the Stiefel manifold, a random center $\MAT{G}_{\St}$ by taking the $\nrank$ first columns of a $\nfeatures\times\nfeatures$ orthogonal matrix uniformly drawn on $\mathcal{O}_{\nfeatures}$.
From there, $\nsamples$ random samples $\MAT{U}_i$ are generated according to $\MAT{U}_i=\expm(\sigma\MAT{\Omega}_i)\MAT{G}_{\St}$, where $\sigma>0$ and $\MAT{\Omega}_i$ is obtained by taking the skew-symmetrical part of a $\nfeatures\times\nfeatures$ matrix whose elements are independently drawn from the centered normal distribution with unit variance.
For Grassmann, the random center $\MAT{G}_{\Gr}$ as well as the random samples $\MAT{P}_i$ are obtained by projecting $\MAT{G}_{\St}$ and $\MAT{U}_i$ on $\Gr$ through~\eqref{eq:st_to_gr}.

To measure the performance on $\St$, we rely on the same similarity measure as in~\cite{kaneko2012empirical}, \textit{i.e.},
\begin{equation}
    \err_{\St}(\MAT{G}_{\St},\MAT{\widehat{G}}) = \| \MAT{G}_{\St}^\top\MAT{\widehat{G}} - \eye_{\nrank} \|_2^2.
    \label{eq:err_st}
\end{equation}
For $\Gr$, we employ the Riemannian distance~\cite{bendokat2024grassmann}, yielding
\begin{equation}
    \textstyle
    \err_{\Gr}(\MAT{G}_{\Gr},\MAT{\widehat{G}}) =
    \|\frac{1}{2} \logm((\eye_{\nfeatures}-\frac12\MAT{G}_{\Gr})(\eye_{\nfeatures}-\frac12\MAT{\widehat{G}}))\|_2^2.
    \label{eq:err_gr}
\end{equation}

Obtained results are displayed in Figures~\ref{fig:stiefel} and~\ref{fig:grassmann}.
Notice that, on $\St$, the results obtained with the $R$-barycenter associated to the orthographic retraction are not displayed since, as expected, it yields the same results as the projected arithmetic mean with the projection based on the polar decomposition (in all considered cases, the difference is lower than $10^{-10}$).
We observe that our proposed projected means perform well on both Stiefel and Grassmann manifolds as compared to other considered barycenters on these simulated data.
On Stiefel, $R$-barycenters based on polar and QR retractions do not perform well as the distance of samples to the mean increases, while our proposed projected arithmetic means remain competitive.


\section{Conclusion and perspectives}
\label{sec:conclusion}

In conclusion, due to their performance and simplicity, our proposed projected arithmetic means appear advantageous as compared to state-of-the-art, computationally expensive, iterative mean estimators both on Stiefel and Grassmann manifolds.
We believe that our approach can also be employed to other manifolds such as the one of symmetric positive semi-definite matrices; see, \textit{e.g.},~\cite{bonnabel2013rank,bouchard2021riemannian}.

\bibliographystyle{unsrt}
\bibliography{biblio}

\appendices
\section*{Supplementary materials}

\subsection{Projection on the Grassmann manifold}
This section contains the proof of Proposition~\ref{prop:gr_proj}.
Due to the structure of the solution, it is better to rely on the representation of Grassmann corresponding to the quotient of the Stiefel manifold $\St$ by the orthogonal group $\mathcal{O}_{\nrank}$~\cite{edelman1998geometry, absil2008optimization}.
In this case, the equivalence class at $\MAT{U}\in\St$ is $\{\MAT{U}\MAT{O}: \, \MAT{O}\in\mathcal{O}_{\nrank}\}$.
The mapping linking this quotient representation of Grassmann to $\Gr$ identified as the rank $\nrank$ projector space~\eqref{eq:gr} is the one in~\eqref{eq:st_to_gr}.
With this parametrization, given $\MAT{X}\in\GrAmbient$, the optimization problem becomes
\begin{equation*}
    \argmin_{\MAT{U}\in\Gr} \quad f(\MAT{U}) = \| \MAT{X} - \MAT{U}\MAT{U}^\top \|_2^2.
\end{equation*}
One can show that the directional derivative of $f$ is
\begin{equation*}
    \begin{split}
        \diff f(\MAT{U})[\MAT{\xi}]
        & = 2 \tr((\MAT{U}\MAT{U}^\top - \MAT{X})(\MAT{U}\MAT{\xi}^\top +\MAT{\xi}\MAT{U}^\top))
        \\
        & = 4 \tr((\MAT{U}\MAT{U}^\top - \MAT{X})\MAT{U}\MAT{\xi}^\top)
        \\
        & = 4 \tr((\eye_{\nfeatures} - \MAT{X})\MAT{U}\MAT{\xi}^\top)
    \end{split}
\end{equation*}
From there, the Euclidean gradient of $f$ is
\begin{equation*}
    \nabla^{\mathcal{E}} f(\MAT{U}) = (\eye_{\nfeatures} - \MAT{X})\MAT{U}.
\end{equation*}
The Riemannian gradient of $f$ on Stiefel $\St$ is then obtained by projecting this Euclidean gradient with the projection~\eqref{eq:st_tangent_proj}, which yields
\begin{equation*}
    \nabla f(\MAT{U}) = (\eye_{\nfeatures} - \MAT{U}\MAT{U}^\top) (\eye_{\nfeatures} - \MAT{X}) \MAT{U}.
\end{equation*}
This also directly corresponds to the Riemannian gradient of $f$ on the quotient representation of Grassmann thanks to the invariance property of $f$ along equivalence classes.
The critical points of $f$ are matrices $\MAT{V}_{\nrank}\in\St$ composed of $\nrank$ eigenvectors of $\MAT{X}$.
Indeed, let $\MAT{V}_{\nrank}\in\St$, $\MAT{V}_{\nfeatures-\nrank}\in\textup{St}_{\nfeatures,\nfeatures-\nrank}$, $\MAT{\Lambda}_{\nrank}\in\mathcal{D}_{\nrank}$ ($\nrank\times\nrank$ diagonal matrices) and $\MAT{\Lambda}_{\nfeatures-\nrank}\in\mathcal{D}_{\nfeatures-\nrank}$ correspond to the eigenvalue decomposition of $\MAT{X}$, \textit{i.e.}, 
\begin{equation*}
    \begin{split}
        \MAT{X}
        & = \MAT{V}_{\nrank}\MAT{\Lambda}_{\nrank}\MAT{V}_{\nrank}^\top + \MAT{V}_{\nfeatures-\nrank}\MAT{\Lambda}_{\nfeatures-\nrank}\MAT{V}_{\nfeatures-\nrank}^\top
        \\
        & =
        \begin{bmatrix}
            \MAT{V}_{\nrank} & \MAT{V}_{\nfeatures-\nrank}
        \end{bmatrix}
        \begin{bmatrix}
            \MAT{\Lambda}_{\nrank} & \MAT{0}
            \\
            \MAT{0} & \MAT{\Lambda}_{\nfeatures - \nrank}
        \end{bmatrix}
        \begin{bmatrix}
            \MAT{V}_{\nrank}^\top \\ \MAT{V}_{\nfeatures-\nrank}^\top
        \end{bmatrix}
    \end{split}
\end{equation*}
One further has $\eye_{\nfeatures} = \MAT{V}_{\nrank}\MAT{V}_{\nrank}^\top + \MAT{V}_{\nfeatures-\nrank}\MAT{V}_{\nfeatures-\nrank}^\top$ and $\MAT{V}_{\nrank}^\top\MAT{V}_{\nfeatures-\nrank}=\MAT{0}$.
It follows that
\begin{multline*}
    \nabla f(\MAT{V}_{\nrank}) = 
    \\
    \MAT{V}_{\nfeatures-\nrank}
    \begin{bmatrix}
        \MAT{0} & \eye_{\nfeatures - \nrank}
    \end{bmatrix}
    \begin{bmatrix}
            \eye_{\nrank} - \MAT{\Lambda}_{\nrank} & \MAT{0}
            \\
            \MAT{0} & \eye_{\nfeatures - \nrank} - \MAT{\Lambda}_{\nfeatures - \nrank}
        \end{bmatrix}
    \begin{bmatrix}
        \eye_{\nrank} \\ \MAT{0}
    \end{bmatrix}.
\end{multline*}
Hence, $\nabla f(\MAT{V}_{\nrank}) = \MAT{0}$, and $\MAT{V}_{\nrank}$ is a critical point.

It remains to determine the set of $\nrank$ eigenvectors of $\MAT{X}$ which yields the minimum.
To do so, let's look at the cost function at $\MAT{V}_{\nrank}$, which is
\begin{equation*}
    f(\MAT{V}_{\nrank})
    = \|\MAT{X} - \MAT{V}_{\nrank}\MAT{V}_{\nrank}^\top\|_2^2
    = \|\MAT{\Lambda}_{\nrank}-\eye_{\nrank}\|_2^2 + \|\MAT{\Lambda}_{\nfeatures-\nrank}\|_2^2.
\end{equation*}
In order to minimize $f$, we must select the eigenvalues such that $\|\MAT{\Lambda}_{\nrank}-\eye_{\nrank}\|_2^2$ is the smallest possible.
These thus correspond to the largest eigenvalues of $\MAT{X}$, which concludes the proof of Proposition~\ref{prop:gr_proj}.

\subsection{$RL$-barycenter on Stiefel based on the QR decomposition}
We consider the $RL$-barycenter from Definition~\ref{def:RLbarycenter} on the Stiefel manifold $\St$ with the retraction based on the QR decomposition defined in~\eqref{eq:st_retr_qr} and with the lifting $L_{\MAT{G}}(\MAT{M})=\diff\qf(\MAT{G})[\MAT{M}-\MAT{G}]$.
Given samples $\{\MAT{M}_i\}_{i=1}^{\nsamples}$, the corresponding fixed-point equation is
\begin{equation*}
    \MAT{G} = \qf\left(\MAT{G} + \frac{1}{\nsamples}\sum_{i=1}^{\nsamples} \diff\qf(\MAT{G})[\MAT{M}_i-\MAT{G}]\right).
\end{equation*}
A solution, if it exists, is such that $\diff\qf(\MAT{G})[\MAT{A} - \MAT{G}]=\MAT{0}$, where $\MAT{A}=\frac{1}{\nsamples}\sum_{i=1}^{\nsamples}\MAT{M}_i$.
Our objective here is to show that the projected arithmetic mean $\MAT{G}=\qf(\frac{1}{\nsamples}\sum_{i=1}^{\nsamples}\MAT{M}_i)$ is a solution.

First, we compute the differential of the QR decomposition in order to get the expression of $\diff\qf(\MAT{G})$.
Let $\MAT{B}=\MAT{Q}\MAT{R}$, where $\MAT{Q}\in\St$ and $\MAT{R}\in\mathcal{T}_{\nrank}$ (space of upper triangular matrices).
It follows that $\diff\MAT{B}=\diff\MAT{Q} \MAT{R} + \MAT{Q}\diff\MAT{R}$, where $\diff\MAT{Q}\in\StTangent{\MAT{Q}}$ and $\diff\MAT{R}\in T_{\MAT{R}}\mathcal{T}_{\nrank}\simeq\mathcal{T}_{\nrank}$ ($\mathcal{T}_{\nrank}$ is a vector space).
From~\cite{absil2008optimization}, we know that there exists $\MAT{\Omega}\in\mathcal{A}_{\nrank}$ (space of skew-symmetric matrices) and $\MAT{K}\in\realspace^{(\nfeatures-\nrank)\times\nrank}$, such that $\diff\MAT{Q}=\MAT{Q}\MAT{\Omega} + \MAT{Q}_{\perp}\MAT{K}$, where $\MAT{Q}_{\perp}\in\textup{St}_{\nfeatures,(\nfeatures-\nrank)}$ is an orthogonal complement of $\MAT{Q}$, \textit{i.e.}, $\MAT{Q}_{\perp}^\top\MAT{Q}=\MAT{0}$.
From there, one gets $\diff\MAT{B}=\MAT{Q}(\MAT{\Omega}\MAT{R}+\diff\MAT{R}) + \MAT{Q}_{\perp}\MAT{K}$.
It follows that $\MAT{Q}^\top\diff\MAT{B}\MAT{R}^{-1}=\MAT{\Omega}+\diff\MAT{R}\MAT{R}^{-1}$ and $\MAT{Q}_{\perp}^{\top}\diff\MAT{B}=\MAT{K}$.
Since $\diff\MAT{R}$, $\MAT{R}\in\mathcal{T}_{\nrank}$, and $\MAT{\Omega}\in\mathcal{A}_{\nrank}$, one can deduce that $\MAT{\Omega}=\tril(\MAT{Q}^\top\diff\MAT{B}\MAT{R}^{-1}) - \tril(\MAT{Q}^\top\diff\MAT{B}\MAT{R}^{-1})^\top$, where $\tril(\cdot)$ cancels the diagonal and upper triangular elements of its argument.
One thus gets
\begin{multline*}
    \diff\qf(\MAT{B})[\diff\MAT{B}] =
    \MAT{Q}_{\perp}\MAT{Q}_{\perp}^\top \diff\MAT{B}
    \\
    + \MAT{Q}(\tril(\MAT{Q}^\top\diff\MAT{B}\MAT{R}^{-1}) - \tril(\MAT{Q}^\top\diff\MAT{B}\MAT{R}^{-1})^\top).
\end{multline*}

We are interested in $\diff\qf(\MAT{G})[\MAT{A} - \MAT{G}]$.
By construction, $\MAT{G}\in\St$ and $\MAT{A}=\MAT{G}\MAT{R}$.
Hence,
\begin{multline*}
    \diff\qf(\MAT{G})[\MAT{A} - \MAT{G}] = \MAT{G}_{\perp}\MAT{G}_{\perp}^\top(\MAT{G}\MAT{R} - \MAT{G})
    \\
    + \MAT{G}(\tril(\MAT{G}^\top(\MAT{G}\MAT{R} - \MAT{G})) - \tril(\MAT{G}^\top(\MAT{G}\MAT{R} - \MAT{G}))^\top).
\end{multline*}
By definition, we have $\MAT{G}_{\perp}^\top\MAT{G}=\MAT{0}$.
Moreover, $\tril(\MAT{G}^\top(\MAT{G}\MAT{R} - \MAT{G}))=\tril(\MAT{R}-\MAT{I}_{\nrank})=\MAT{0}$.
It is enough to conclude that $\MAT{G}=\qf(\frac{1}{\nsamples}\sum_{i=1}^{\nsamples}\MAT{M}_i)$ is indeed an $RL$-barycenter on $\St$ with the retraction based on the QR decomposition defined in~\eqref{eq:st_retr_qr} and with the lifting $L_{\MAT{G}}(\MAT{M})=\diff\qf(\MAT{G})[\MAT{M}-\MAT{G}]$.

\end{document}